\documentclass[sn-mathphys-num]{sn-jnl}

\usepackage{graphicx}%
\usepackage{multirow}%
\usepackage{amsmath,amssymb,amsfonts}%
\usepackage{amsthm}%
\usepackage{mathrsfs}%
\usepackage[title]{appendix}%
\usepackage{xcolor}%
\usepackage{textcomp}%
\usepackage{manyfoot}%
\usepackage{booktabs}%
\usepackage{algorithm}%
\usepackage{algorithmicx}%
\usepackage{algpseudocode}%
\usepackage{listings}%
\usepackage{lineno}

\theoremstyle{thmstyleone}%
\newtheorem{theorem}{Theorem}
\newtheorem{proposition}[theorem]{Proposition}%

\theoremstyle{thmstyletwo}%
\newtheorem{example}{Example}%

\theoremstyle{thmstylethree}%

\begin{document}

\title[Article Title]{A Novel Task-Driven Method with Evolvable Interactive Agents Using Event Trees for Enhanced Emergency Decision Support}

\author[1]{\fnm{First} \sur{Xingyu Xiao}}\email{xxy23@mails.tsinghua.edu.cn}

\author[2]{\fnm{Second} \sur{Peng Chen}}\email{ chenpeng23@mails.ucas.ac.cn}

\author[1]{\fnm{Third} \sur{Ben Qi}}\email{qib22@mails.tsinghua.edu.cn}

\author*[1]{\fnm{Fourth} \sur{Jingang Liang}}\email{jingang@tsinghua.edu.cn}

\author[1]{\fnm{Fifth} \sur{Jiejuan Tong}}\email{tongjj@tsinghua.edu.cn}

\author[1]{\fnm{Sixth} \sur{Haitao Wang}}\email{wanght@tsinghua.edu.cn}

\affil*[1]{\orgdiv{Institute of Nuclear and New Energy Technology}, \orgname{Tsinghua University}, \orgaddress{\street{Shuangqing Road}, \city{Beijing}, \postcode{100084}, \state{Beijing}, \country{China}}}

\affil[2]{\orgdiv{Institute of Software}, \orgname{Chinese Academy of Sciences}, \orgaddress{\street{Zhongguancun South Fourth Street}, \city{Beijing}, \postcode{100190}, \state{Beijing}, \country{China}}}


\abstract{As climate change and other global challenges increase the likelihood of unforeseen emergencies, the limitations of human-driven strategies in critical situations become more pronounced. Inadequate pre-established emergency plans can lead operators to become overwhelmed during complex systems malfunctions. This study addresses the urgent need for agile decision-making in response to various unforeseen incidents through a novel approach, EvoTaskTree (a task-driven method with evolvable interactive agents using event trees for emergency decision support). This advanced approach integrates two types of agents powered by large language models (LLMs): task executors, responsible for executing critical procedures, and task validators, ensuring the efficacy of those actions. By leveraging insights from event tree analysis, our framework encompasses three crucial tasks: initiating event subevent analysis, event tree header event analysis, and decision recommendations. The agents learn from both successful and unsuccessful responses from these tasks. Finally, we use nuclear power plants as a demonstration of a safety-critical system. Our findings indicate that the designed agents are not only effective but also outperform existing approaches, achieving an impressive accuracy rate of up to 100 \% in processing previously unencountered incident scenarios. This paper demonstrates that EvoTaskTree significantly enhances the rapid formulation of emergency decision-making.}





\maketitle
\section{Introduction}\label{sec1}

In the face of escalating global challenges, such as climate change, the frequency and severity of unforeseen emergencies are expected to rise\cite{glock2015decision}. This trend underscores the inherent limitations of traditional, human-driven strategies in managing critical situations, particularly in the context of safety-critical industries like nuclear power plants\cite{bhattacharya2015green}. The inadequacy of pre-established emergency response plans often leads to overwhelming circumstances for operators, especially during system malfunctions where swift and accurate decision-making is crucial.

The traditional decision-support method is multi-criteria decision-making (MCDM). Although it is highly effective, the method is time-consuming and complex\cite{taticchi2015review, Crichton}, making it incapable of providing a rapid response to unforeseen situations. The popular real-time online decision support system for nuclear power plant emergencies (RODOS) has utilized this method for decision-making \cite{Fertier}. This method often includes four steps, construct a criteria system and case description, determine suitable alternatives for achieving the goals of the problems, and make final decision. Specific methods of MCDM include analytic hierarchy process (AHP) \cite{Liu2020}, technique for order preference by similarity to ideal solution (TOPSIS) \cite{Shih}, preference ranking organization method for enrichment evaluations (PROMETHEE) \cite{Venkata}, simple additive weighting (SAW) \cite{Simanaviciene}, visekriterijumsko kompromisno rangiranje (VIKOR) \cite{Cristóbal}, complex proportional assessment (COPRAS) \cite{Roy2019}, and multi-attribute utility theory (MAUT) \cite{Figueira} and so on. Applications of MCDM in NPPs are primarily concentrated on site selection, post-incident recovery, and off-site emergency response\cite{Xiao2024}. Moreover, to ensure the safety of NPPs, experts have developed a series of works, including qualitative risk assessment, and emergency plans. However, there is currently a scarcity of work that integrates both approaches\cite{Xiao2024}. 

The current approaches for decision support remain inherently subjective. This subjectivity arises from the fact that emergency plans are ultimately derived from human-generated content, making them susceptible to certain limitations. Large language models (LLMs) have consistently outperformed average human benchmarks across various metrics\cite{kasneci2023chatgpt}. They leverage vast databases to function as comprehensive knowledge repositories. Additionally, They have increasingly been deployed as expert assistants in various fields, such as medicine \cite{thirunavukarasu2023large}, education \cite{kasneci2023chatgpt}, and law \cite{cui2023chatlaw}. Thus, we aim to capitalize on the extensive knowledge embedded in LLMs to minimize the inherent subjectivity in the formulation of emergency plans.

In this study, we develop a zero-shot strategy named EvoTaskTree (a task-driven method with evolvable interactive agents using event trees for enhanced emergency decision support) that covers three tasks: two tasks in event tree analysis and decision recommendations. There are mainly two types of agents in the simulacrum: task executors and validators. These agents will perform three tasks for different emergency initial events by referencing a comprehensive successful and unsuccessful experience base. Finally, we use a nuclear power plant as a demonstration of a safety-critical system.  An overview of the simulated environment is shown in Figure~\ref{inturface}. It aids operators in effectively managing new types of incidents, providing more effective support for safety-critical industries decision-making.

\begin{figure}[H]
\centering
\includegraphics[width=0.9\textwidth]{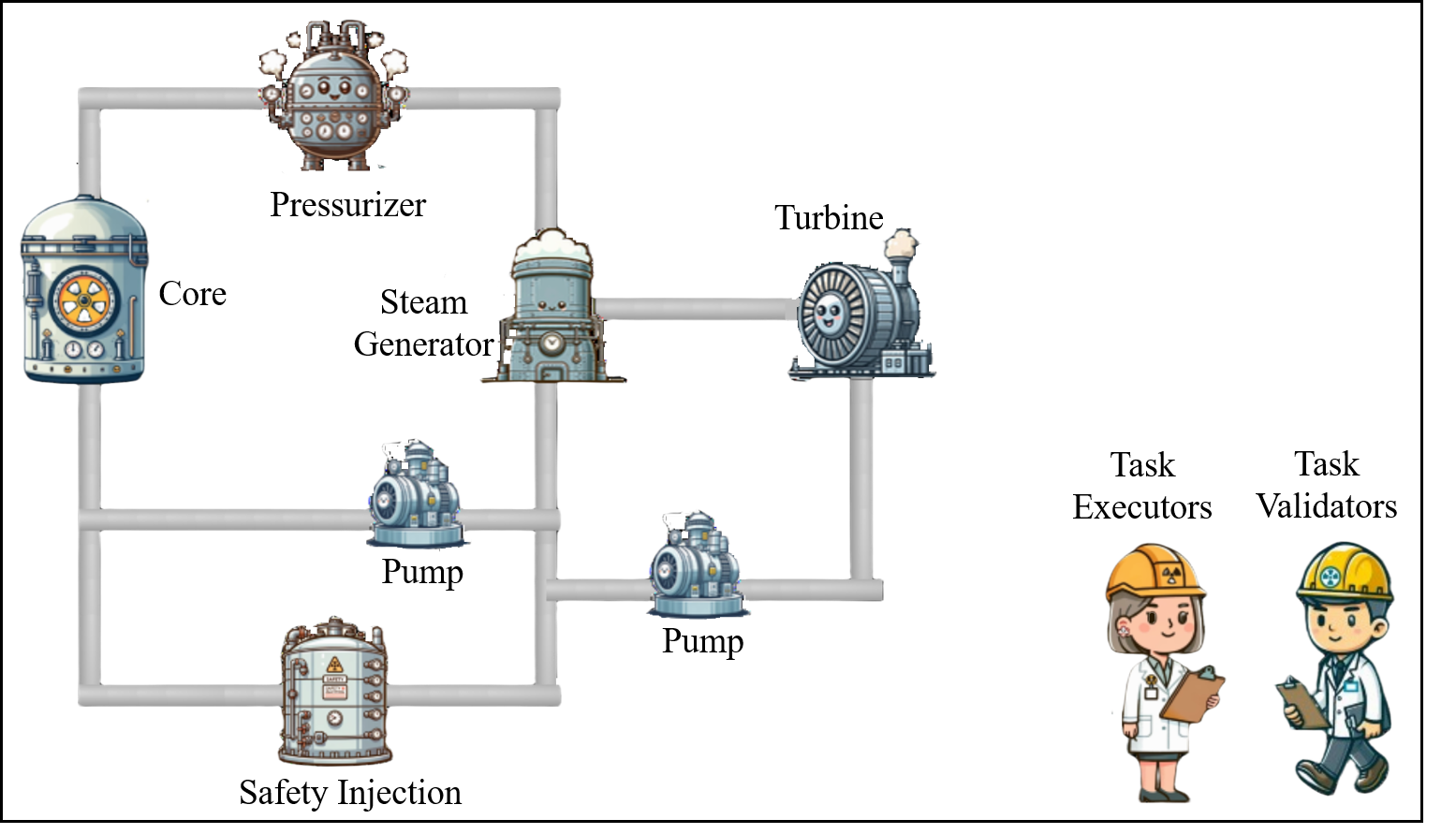}
\caption{{An overview of NPP simulacrum. The operators are autonomous agents powered by large language models. An interesting finding is that the agents can keep improving performance over time without manually labeled data.}}
\label{inturface}
\end{figure}

The main contributions of our work are summarized as follows:
\begin{itemize}
\item We have constructed a task-driven virtual platform for complex systems, which includes two types of agents powered by large language models (LLMs).
\item We integrate event tree knowledge from qualitative risk assessment with emergency decision-making. This integration aids in the analysis of emergency decision-making outcomes.
\item We propose a novel approach that can rapidly and effectively respond to unforeseen emergencies, while being minimally reliant on human subjectivity.
\end{itemize}

\section{Event Tree Theory in Emergency Decision Support}\label{sec10}

This section introduces the foundational theories of event trees, providing an overview of their basic principles. Additionally, it will explore how event trees are integrated with emergency decision support systems, illustrating the methodologies and benefits of combining these tools in the context of emergency response and management.

\begin{theorem}[Event tree analysis]\label{thm1}
Event tree analysis (ETA) is frequently used in conjunction with fault tree analysis (FTA) for quantitative risk analysis (QRA). This method collaboratively develops a logical relationship among the events leading to an accident and estimates the associated risk. In the event tree, the unwanted event is named as an initiating event, and the follow-up consequences are termed as events or safety barriers \cite{ferdous2011fault}.

Specifically, ETA is a technique used to describe the consequences of initiating an event and estimate the likelihood (frequency) of possible outcomes of the event. The ETA represents the dichotomous conditions (e.g., success/ failure, true/false, or yes/no) of the initiating event until the subsequent events lead to the outcome events. Kabir \cite{kabir2017overview} and Lees \cite{lees2012lees} provide a detailed procedure for constructing and analyzing the ETA for a process system.

With an assumption that the events are independent, deterministic and probabilistic approaches use the equation \ref{eq1} to analyze event trees. $P_{i}$ denotes the probability of $i$th ($i$ = 1, 2, 3, ..., n) events, $\lambda $ denotes the frequency for the initiating event, $\lambda_{i} $ denotes the frequency for outcome events, and $j$ denotes the number of success events.
\begin{equation}
\lambda _{i}=\lambda  \times \prod_{i=1}^{j}P_{i}  \times \prod_{i=j}^{n}(1-P_{i})
\end{equation}

\end{theorem}

Next, a simple system from a nuclear power plant is introduced as a case study to demonstrate the characteristics of fault tree analysis.

\begin{example} \label{example1}

A simplified nuclear power system is illustrated in figure \ref{system}, which primarily consists of two subsystems: the safe injection system and the containment spray system. The core meltdown can only be prevented when both subsystems are functioning properly simultaneously.

\begin{figure}[H]
\centering
\includegraphics[width=0.9\textwidth]{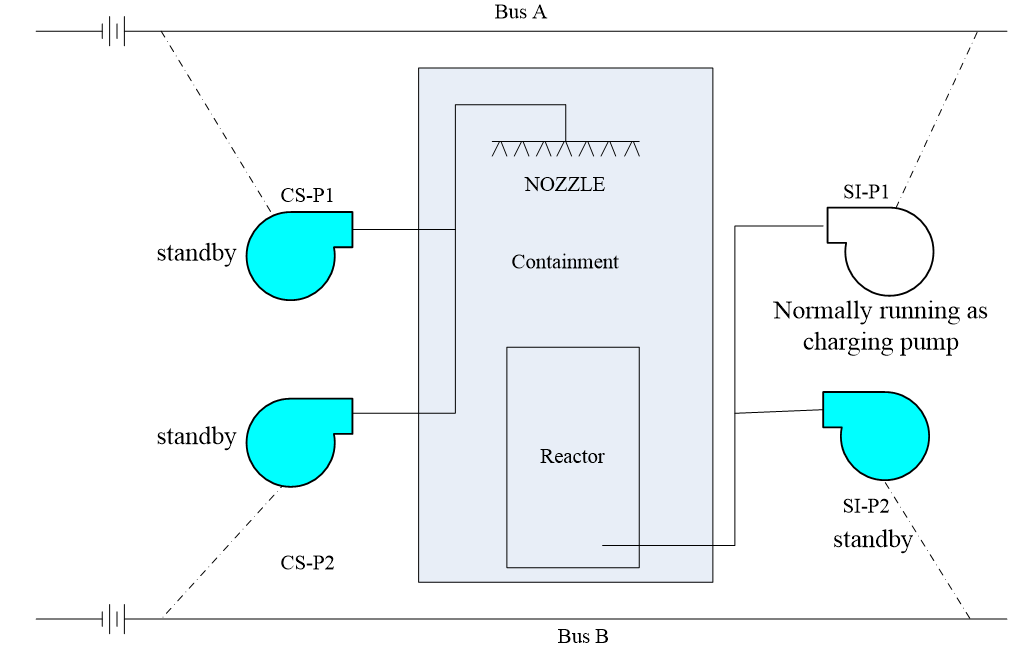}
\caption{{Simplified System Diagram for an NPP.}}
\label{system}
\end{figure}

Figure \ref{ET} illustrates the event tree analysis of the simplified nuclear power system, where a Large LOCA is the initiating event (IE). The first header event ($E_{1}$) is the normal operation of the safe injection system, and the second header event ($E_{2}$) is the proper functioning of the containment spray system. Depending on the sequence of successes or failures of these header events, the event tree leads to different consequences, including "No core meltdown" and "Core meltdown". By applying the event tree theory, the probability of the final consequence ($\lambda _{i}$) is calculated based on the probability of the IE ($\lambda$) and the probabilities of each header event ($P_{i}$).

\begin{figure}[H]
\centering
\includegraphics[width=0.9\textwidth]{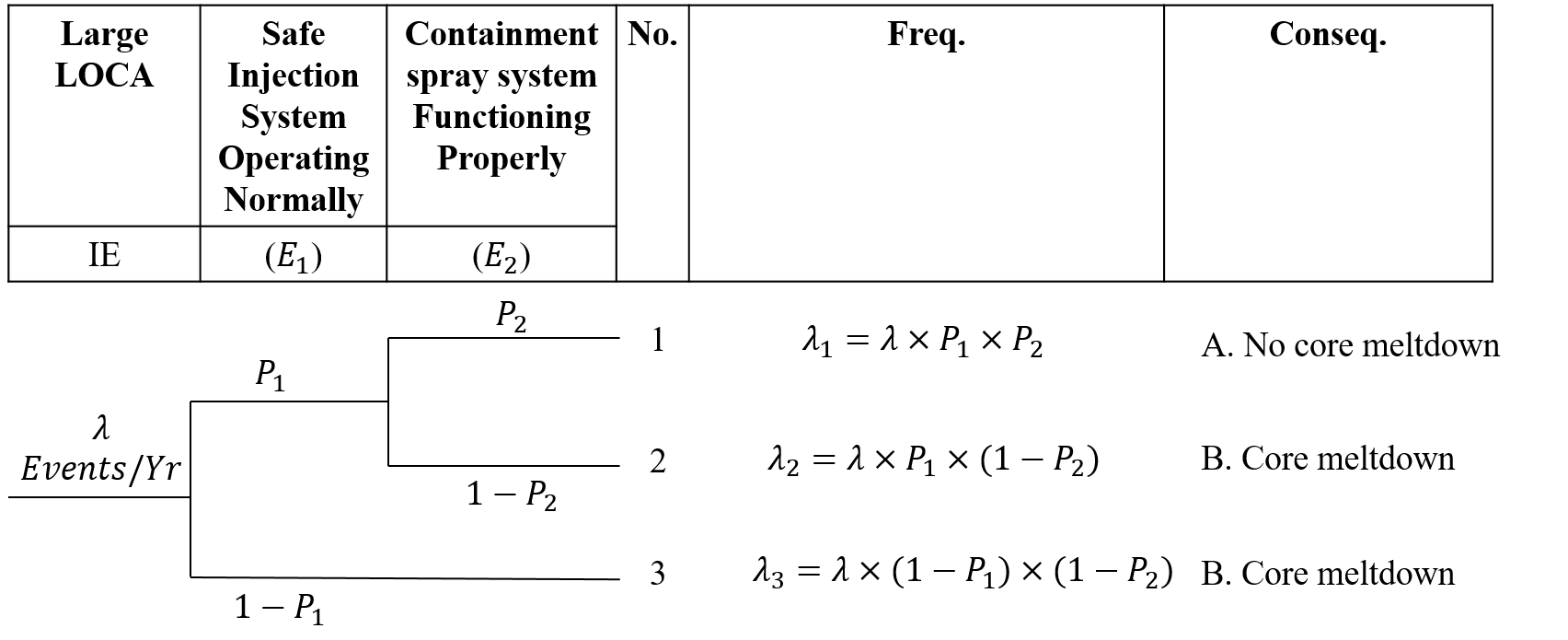}
\caption{{Event tree for Large LOCA.}}
\label{ET}
\end{figure}

\end{example}

\begin{proposition}
The emergency decision support problem in the face of unexpected situations can be translated into the problem of rapidly constructing an event tree in the face of unexpected situations.
\end{proposition}

\begin{proof}
In responding to emergencies at an NPP, the ultimate goal of emergency measures is to prevent the core meltdown. The event tree provides us with the progression of events within the system under different conditions when facing unexpected situations. Therefore, if our proposed emergency plan can guide each header event towards a successful consequence, we can completely prevent core meltdown and effectively manage unexpected event. 

For example \ref{example1}, if through our emergency response, we set $P_{1}$ and $P_{2}$ to 1, then $1-P_{1}$ and $1-P_{2}$ will become 0. This results in a zero probability of a core meltdown. It indicates that our emergency measures are effective. Therefore, when dealing with an unforeseen emergency, if we can obtain information from the event tree accurately and effectively, we can develop robust emergency response strategies. Consequently, we trace the task of emergency decision support back to the analysis of initiating events subevents within the event tree, as well as the analysis of header events.
\end{proof}

The upper portion of the diagram~\ref{pipe} illustrates the integration of event tree theory with emergency decision support. This integration completes a closed-loop analysis from the initial event to the final decision support. Additionally, this diagram details the input and output at each step of our EvaTaskTree, encompassing initial event and sub-event analysis, event tree header events, and decision recommendations. In summary, integrating ETA with the emergency decision support process not only enables the proposal of effective emergency measures but also allows for the reutilization of prior knowledge generated during the risk assessment phase, thereby fully leveraging the results of ETA.

\begin{figure}[H]
\centering
\includegraphics[width=0.9\textwidth]{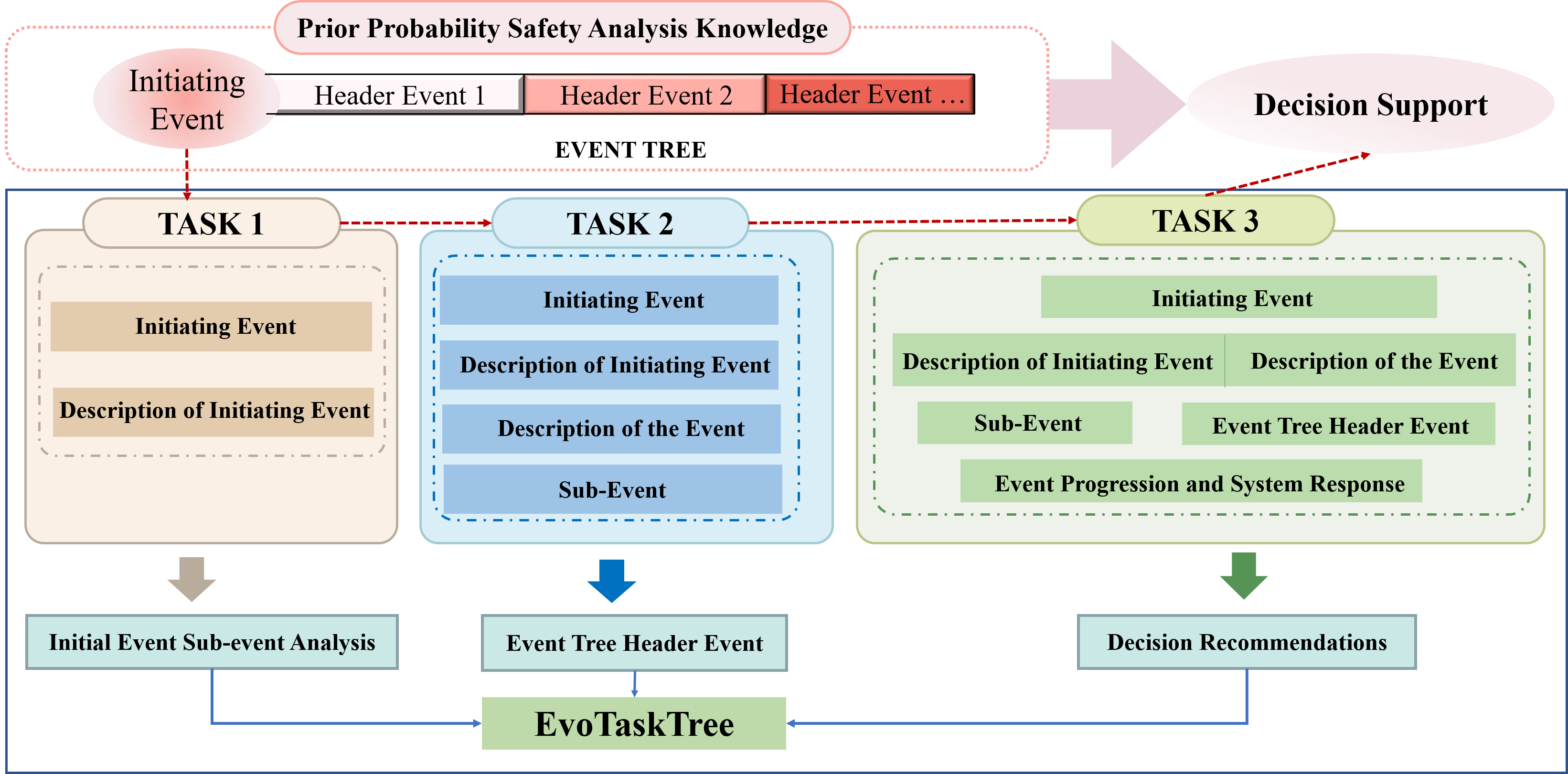}
\caption{{The integration of Event Tree Theory with emergency decision support.}}
\label{pipe}
\end{figure}

\section{Methodology}

Evolution is a critical feature of our EvaTaskTree. This section will provide a detailed explanation of the evolution process, along with specific implementation details regarding inference and evaluation.

\subsection{Evolution}
\label{pipe}

The whole process of our EvoTaskTree is shown in Figure ~\ref{pipeline} (a). EvoTaskTree is a parameter-free strategy. For each task, the task executors and task validators interact and iterate continuously to complete it. The three tasks are interconnected through a task flow, collectively fulfilling the decision-support function. Figure ~\ref{pipeline} (b) provides a detailed description of the agent-agent interaction. There are two important modules in this strategy, namely the record library and the experience base. Successful cases, which are to be used as references for future interventions, are compiled and stored in the record library. For cases where treatment fails, the results will be added to the experience base. In the three task simulations, we employ dense retrievers to retrieve related historical records and guiding principles, assisting agents in delivering better results. As experience and records are accrued, they are actively applied, with both the record library and the experience base being perpetually updated. Additionally, the determination of correctness is autonomously completed by the agent. The complete code is available on the GitHub website (https://github.com/Crystalxy123/EvoTaskTree/tree/master).

\begin{figure}[H]
\centering
\includegraphics[width=0.9\textwidth]{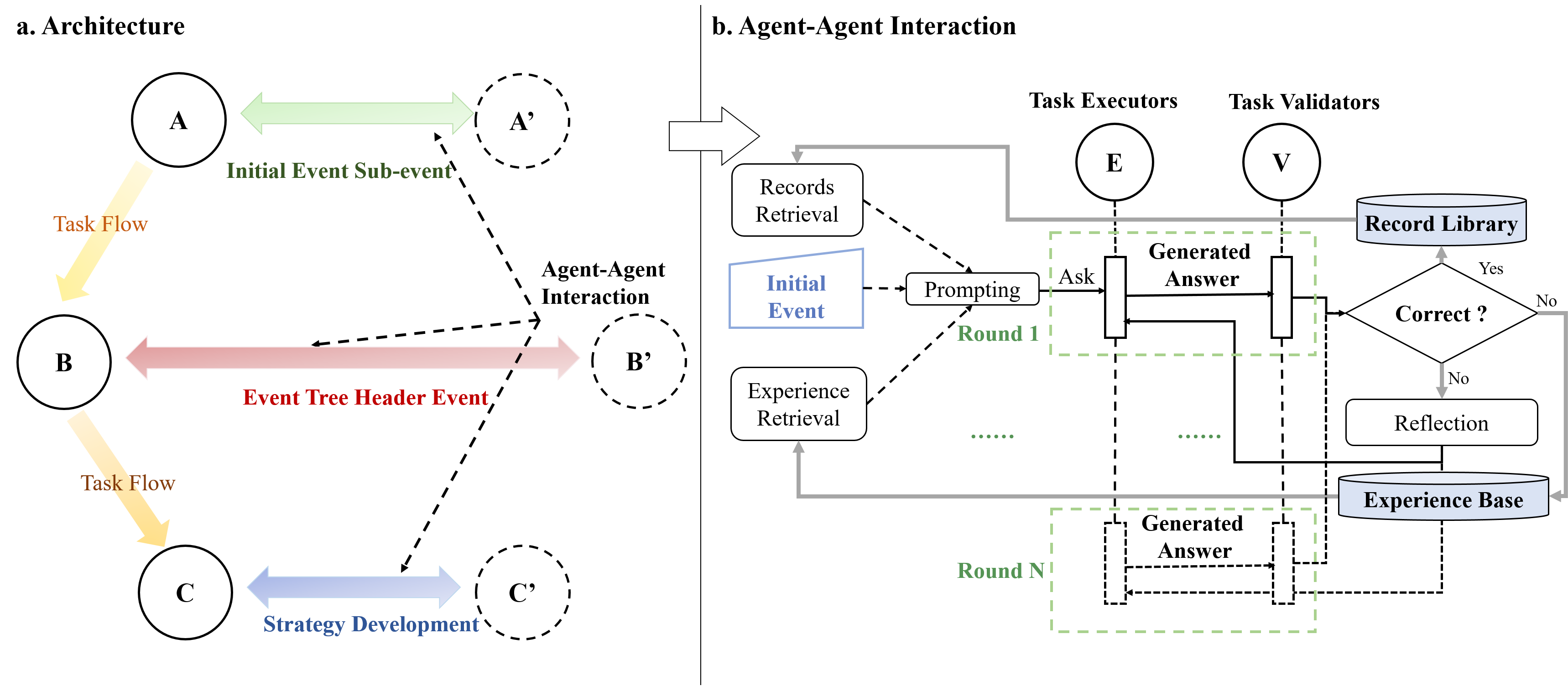}
\caption{{The overview of the EvoTaskTree. (a) Architecture. It is a parameter-free strategy. For each task, the task executors and task validators interact and iterate continuously to complete tasks. The three tasks are interconnected through a task flow, collectively fulfilling the decision support. (b) Agent-Agent Interaction. This diagram illustrates the methods to achieve self-evolution: 1) accumulating examples; 2) adding correct responses directly to the record library; 3) adding incorrect responses directly to the experience base; 4) utilizing both knowledge to retrieve the most similar content for reasoning during the inference process. 5) During the training process, consistently repeat steps 1, 3, and 4 until the correct response is identified;
6) During the testing process, directly apply step 4 to generate the final answer.}}
\label{pipeline}
\end{figure}

\textbf{Record Library Building.} During the process of completing the tasks, it is highly beneficial for operators to reference previously validated records. These records contain abundant knowledge and demonstrate the rationale behind accurate and adequate responses to diverse emergencies. Therefore, we propose to build record libraries for agents to sharpen their abilities. The library is structured in the format of question-answer pairs, where the question details the fault condition requiring task completion, and the answer contains the validated response. As shown in the upper part of Figure~\ref{pipeline}, for each generated answer from agents, the question-answer pair will be added to the record library if the answer is correct. When a new query is coming, it will search for related records from the library based on dense retrieval techniques using \cite{nori2023}. Each task has its private record library and experience base to avoid irrelevant records utilization. For the simulated three tasks, we aim to enable agents to gain professional experience from historical records. 

\textbf{Experience Base Expanding.} Learning from errors is also crucial for the growth of operators. We believe that LLM-powered professional agents can engage in self-reflection from these errors, distilling relevant principles (experience) to ensure correct task completion when encountering similar issues in future cases. We draw inspiration from a previous study \cite{zhang2024} to allow doctor agents to learn from failures. As shown in the below part of Figure ~\ref{pipeline}, if the answer is wrong, the question-answer pair be added to the experience base. To eliminate the influence of noise and maximize the utilization of the experience base, we incorporate additional judgment when utilizing experience. This judgment involves evaluating whether the most similar experience retrieved based on semantic similarity. 

\subsection{Inference}
\label{inference}
Based on the record library and experience base introduced above, we enhance the prompt for the agents using successful and unsuccessful responses. We get the most similar ones by comparing existing queries in the record library and experience base with the current query. Then, the chosen successful and unsucessful question-answer pair are arranged for few-shot examples in the prompt. Both records and experience are retrieved using cosine similarity and text is embedded into vector space by the "text-embedding-ada-002" model provided by OpenAI.

\subsection{Evalution}
As for the evaluation, to ensure consistency in evaluation, the results of all test sets were manually annotated by experts. Due to the unique characterization of safety-critical industries, we aim to achieve 100\% accuracy in our results. Consequently, we categorized all results into two types: correct cases and incorrect cases. The accuracy is then determined by the proportion of correct cases within the test set.

\section{LLM-based EvoTaskTree Simulacrum in Nuclear Power Plant}

To effectively demonstrate the applicability of our EvoTaskTree in industrial settings, we have chosen nuclear power plants (NPPs) as a representative example. NPPs play a crucial role in the global energy landscape due to their capacity to produce large amounts of electricity with relatively low greenhouse gas emissions. Many countries adopt a positive attitude toward nuclear energy. Given the severe consequences of potential accidents in NPPs, it is crucial to support operators in making swift decisions during emergency situations.

\subsection{Environment Settings}

We first develop a comprehensive NPP simulation environment inspired by previous studies\cite{Hospital, Software}. The NPP sandbox is implemented using Pygame\cite{Pygame}, a cross-platform library of Python modules specifically designed for the development of video games. Finally, as shown in Figure\ref{inturface}, it illustrates a simplified NPP system. The overall system functions by transferring heat from the reactor core, where nuclear fission takes place. Heat is transferred through the pressurizer to maintain system pressure. The heat then moves to the steam generator, where water is converted to steam. Finally, the steam is sent to the turbine, where it drives mechanical power generation. The pumps ensure a continuous flow of coolant, and the safety injection system acts as a critical safeguard mechanism.

\subsection{Definition of Decision Support Tasks}
To support emergency strategies for industries, as we illustrated in section\ref{sec10}, we have designed three interdependent tasks: initiating event subevent analysis, event tree header event analysis, and decision recommendations. These tasks are closely related and collectively form an essential part of the industries response system. Task1 primarily focuses on the reclassification of events. Task2 requires listing the header events, including specific event information, and ensuring the sequence of events follows their developmental progression. Task3 is our final critical task,  providing decision support. The contents of each task are as follows.

\begin{itemize}
\item Task1. Initiating event subevent analysis identifies. This provides a foundational analysis for task2 and task3. For task1, the inputs are the initial event name and description, and the output is the subevent of the initial event.

\item Task2. Event tree header event analysis involves identifying and listing header events that could lead to accidents. This task analyzes the subsequent events systematically. The inputs are the initial event, initial event description, subevent (generated from task1), and event process and system response, while the output is the event tree header event.

\item Task3. Decision recommendations involves developing and providing specific operational decisions to address and mitigate the adverse impact after identifying and analyzing the potential subevents and header events. The inputs are initial event, initial event description, subevent (generated from task1), event process and system response, and event tree header event (generated from task2), while the output is operator action.
\end{itemize}

\subsection{Agent Roles}
\label{prompt}

We designed two types of roles for each task in the NPP simulacrum, including initial event subevent analysis agent, initial event subevent verification agent, event tree header event analysis agent, event tree header event verification agent, strategy development agent, and strategy validation agent. 

\textbf{Initial Event Subevent Analysis Agent.} The initial event subevent analysis agent identifies subevents of initiating events. The agent's primary task is to examine detailed descriptions and cases, both correct and incorrect. For incorrect cases, the agent must summarize feedback to prevent recurrence. For correct cases, the agent should use the logic applied. When correct and incorrect cases are unavailable, the agent must analyze the subevents of initiating events. 

\textbf{Initial Event Subevent Verification Agent.} 
The initial event subevent validation agent is responsible for verifying the accuracy of subevents identified by the analysis agent. Based on the initial event and its detailed description, the agent determines whether the identified subevents match the actual subevents. The criteria for judgment include several key aspects. This ensures a thorough and accurate validation process.

\textbf{Event Tree Header Event Analysis Agent.} The event tree header event analysis agent specializes in identifying and listing the header events in event trees. The agent must also summarize feedback for incorrect cases to avoid repeated mistakes and reference the logic used in correct cases. Continuous improvement is encouraged by using feedback from incorrect cases to enhance future responses.

\textbf{Event Tree Header Event Verification Agent.} The event tree header events verification agent validated the accuracy of event tree header events provided by the header event analysis agent. This agent ensures that the identified header events match the actual event progression within NPPs. The verification process involves several key steps to maintain the reliability and integrity of the analysis. This contributes significantly to the accuracy and safety of operations.

\textbf{Strategy Development Agent.} 
The strategy development agent in NPPs is responsible for determining specific actions that operators should execute based on initial events, their descriptions, subevents, the sequence of events, and examples of correct and incorrect responses. The agent must ensure that the actions are detailed and accurate. For incorrect cases, the agent needs to summarize feedback to prevent recurrence, using the logic from correct cases to inform the analysis. 

\textbf{Strategy Validation Agent.} 
The strategy validation agent is responsible for assessing the recommendations provided by NPP strategy advisors. This agent determines if the recommendations align with the actual operator's suggestions. It ensures the safety and reliability of NPP operations through accurate validation of strategic advice.

\section{Simulation and Results}

\subsection{Experimental Setting}

\textbf{Dataset.} Our comprehensive dataset distribution is illustrated in Table ~\ref{dataset}. The dataset encompasses 11 distinct types of common initiating events observed in NPPs: loss of coolant accident (LOCA) with 5 instances, loss of heat sink accident (LOHSA) occurring 3 times, loss of feedwater event (LOFW) noted twice, and loss of offsite power (LOOP) occurring once. Additionally, anticipated transient without scram (ATWS) is the most frequent event, recorded 7 times. The dataset also includes 2 instances each of main feedwater line break (MFLB) and steam generator tube rupture (SGTR). The main steam line break (MSLB) has 6 occurrences, whereas the loss of DC power accident (LODC) is recorded 3 times. Moreover, the combination of main steam line break and steam generator tube rupture (MSLB+SGTR) is observed 5 times. Lastly, the secondary loop transient event (SLTE) is noted twice. For task1, there are 13 data points in total, with 10 used for training and 3 for testing. For task2 and task3, there are 38 data points in total, with 31 used for training and 7 for testing.

\begin{table}[h]
\caption{Frequency Distribution and Acronyms of Initial Events}\label{dataset}%
\begin{tabular}{@{}lll@{}}
\toprule
 Initial Event                            &  Acronym & Num \\ 
\midrule
Loss of Coolant Accident                                         & LOCA                            & 5                           \\
Loss of Heat Sink Accident     & LOHSA                           & 3                           \\
Loss of Feedwater Event                                          & LOFW                            & 2                           \\
Loss of Offsite Power                                            & LOOP                            & 1                           \\
Anticipated Transient Without Scram                              & ATWS                            & 7                           \\
Main Feedwater Line Break                                        & MFLB                            & 2                           \\
Main Steam Line Break                                            & MSLB                            & 6                           \\
Loss of DC Power Accident                                        & LODC                            & 3                           \\
Steam Generator Tube Rupture                                     & SGTR                            & 2                           \\
Main Steam Line Break Combined with Steam Generator Tube Rupture & MSLB+SGTR                       & 5                           \\
Secondary Loop Transient Event                                   & SLTE                            & 2                           \\ 
\botrule
\end{tabular}
\footnotetext{Note: The table lists the full name, acronym, and the number of occurrences of different incidents.}
\end{table}

\textbf{Implementation Details.} For each query, the number of utilized records and principles after retrieving is set to 1, i.e., only the top 1 relevant experience and records are adopted in the prompt. The record library and experience base are training from empty and will be updated dynamically during training to support further decisions. So the training of an NPP strategy development agent is similar to a new operator improving her/his skills by practicing. All of our simulation experiments leverage GPT-4o as the framework’s backbone model.

\subsection{Experience Accumulation}
\label{Experience Accumulation}
During the iterative process, we designed a validation agent that provides feedback both with and without reasons to discuss the impact of these reasons on accuracy. As illustrated in figure~\ref{task}, it depicts the growth in the number of experience-based and record library instances for three tasks as the number of training samples increases, under conditions with and without reasons.  Specifically, it refers to whether there are reasons when referencing experience-based and record library instances. 

\begin{figure}[H]
\centering
\includegraphics[width=0.9\textwidth]{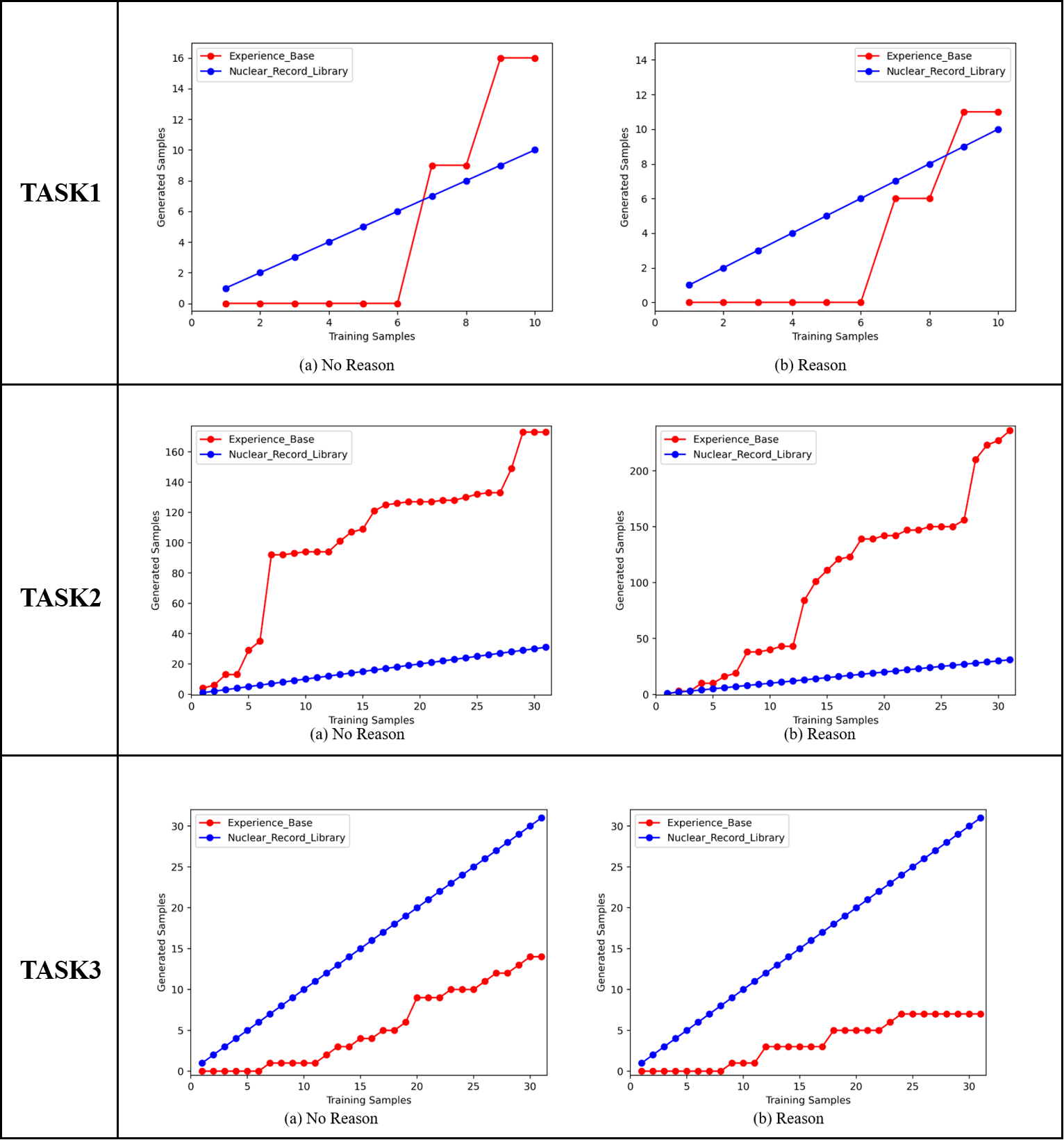}
\caption{{The number of accumulated principles and wrong answers of task2 with the increase of training samples.}}
\label{task}
\end{figure}

In subfigure ~\ref{task} task1 (a), it is observed that the experience base initially generates no samples up to six training samples, after which there is a significant increase, peaking at 16 samples with nine training samples. Conversely, the record library demonstrates a more consistent linear growth, steadily increasing the number of generated samples as the number of training samples increases. Subfigure ~\ref{task} task1 (b) shows a similar pattern. The experience base again shows no sample generation until six training samples, followed by a sharp increase. The record library continues to exhibit linear growth, maintaining a steady increase in generated samples corresponding to the increase in training samples.

In subfigure ~\ref{task} task2 (a), the experience base shows a gradual increase in the number of generated samples, with noticeable fluctuations, culminating at approximately 14 samples when the training samples reach 30. In contrast, the record library method demonstrates steady linear growth, consistently increasing the number of generated samples in direct proportion to the increase in training samples, reaching 31 samples when the training samples are 31. Subfigure ~\ref{task} task2 (b) reveals a similar trend. The experience base again shows a gradual increase with fluctuations, reaching around 7 samples at 30 training samples. Meanwhile, the record library maintains its linear growth pattern, achieving 31 samples at 31 training samples.

In subfigure ~\ref{task} task3 (a), the experience base demonstrates a non-linear growth pattern, with an initial slow increase followed by a sharp rise around six training samples, continuing with fluctuations and ultimately reaching approximately 165 samples at 30 training samples. On the other hand, the record library shows a consistent, linear growth, steadily increasing the number of generated samples, peaking at about 31 samples with 31 training samples.
Subfigure ~\ref{task} task3 (b) reveals a similar trend. The figure illustrates the growth in the number of generated samples for experience base, and record library, in the context of task3 as the number of training samples increases. The graph displays a non-linear growth pattern for the experience base. Initially, there is a slow increase in generated samples, followed by a significant rise beginning around six training samples. This method continues to show fluctuations with subsequent sharp increases, ultimately reaching approximately 250 samples at 31 training samples. In contrast, the NPP record library demonstrates a consistent, linear growth pattern. The number of generated samples steadily increases with the number of training samples, reaching around 31 samples at 31 training samples.

\subsection{Experimental Results}
\label{Experimental Results}
In this section, we discuss the performance of the experience base and the record library on the test dataset. We evaluate both systems in scenarios with and without reasoning of the record library and the experience base. These results correspond to three different tasks. We analyze the proportion of correct feedback within these results. Additionally, we examine the impact of providing reasons during testing. We investigate how the presence or absence of reasoning affects the agents' effectiveness.

As shown in Figure ~\ref{subevent} (a), the line representing feedback with reasoning consistently maintains a higher accuracy across almost all points, indicating that the inclusion of reasoning enhances the validation agents' performance. The accuracy of feedback without reasoning exhibits significant fluctuations. Conversely, the reasoning-included feedback line shows less variance and maintains a relatively stable and higher accuracy throughout, except for a notable dip around sample 6. Additionally, both systems ultimately achieved an accuracy of 1.

As shown in Figure ~\ref{subevent} (b), the feedback accuracy without reasoning starts high, dips significantly at the fifth sample, and then shows a pattern of recovery and stability. In contrast, the feedback accuracy with reasoning starts much lower but shows a dramatic increase to reach perfect accuracy by the eighth sample and maintains this level through the tenth sample.

\begin{figure}[H]
\centering
\includegraphics[width=0.9\textwidth]{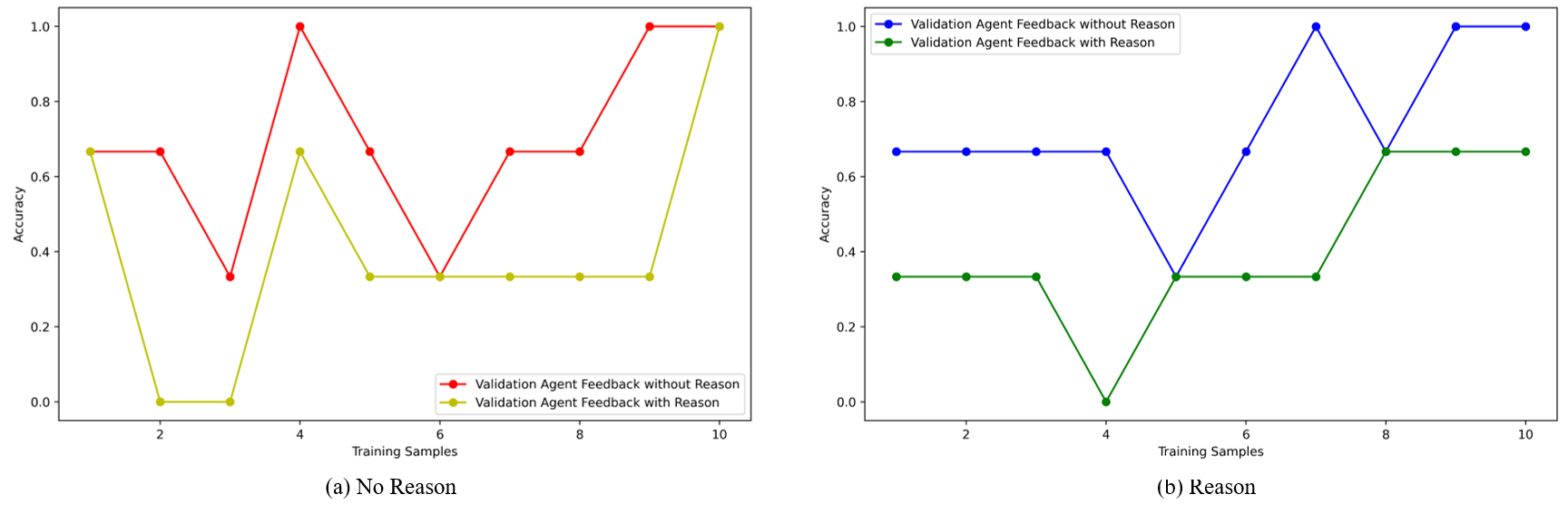}
\caption{{Task1 Experimental Results: Comparative Analysis of Validation Agent Feedback Accuracy Without Reasoning Across Training Samples.}}
\label{subevent}
\end{figure}

Results show that removing reasoning allows for more rapid attainment of high accuracy during the infer stage, as seen in the quicker ascent to high accuracy in the early stages in Figure~\ref{subevent} (a). However, the absence of reasoning in training data leads to significant fluctuations throughout the evaluation process. In contrast, incorporating reasoning, as depicted in Figure~\ref{subevent} (b), significantly enhances the stability of accuracy rates. This suggests that while omitting reasoning can yield quick initial results during the infer stage, adding reasoning in training data is crucial for achieving sustained and stable high accuracy in model performance during testing phases.

As for task2, Figure~\ref{header} (a) illustrates the accuracy of validation agent feedback across 31 training samples, differentiated by whether reasoning was included in the feedback process. The graph~\ref{header} (b) presents the accuracy of validation agent feedback across 31 training samples, comparing performances with and without the inclusion of reasoning. These figures suggest that removing reasoning plays a role in moderating the fluctuations in the accuracy of feedback. This stability is particularly noticeable during the latter half of the dataset, where the red line demonstrates a higher level of accuracy compared to the yellow line.

\begin{figure}[H]
\centering
\includegraphics[width=0.9 \textwidth]{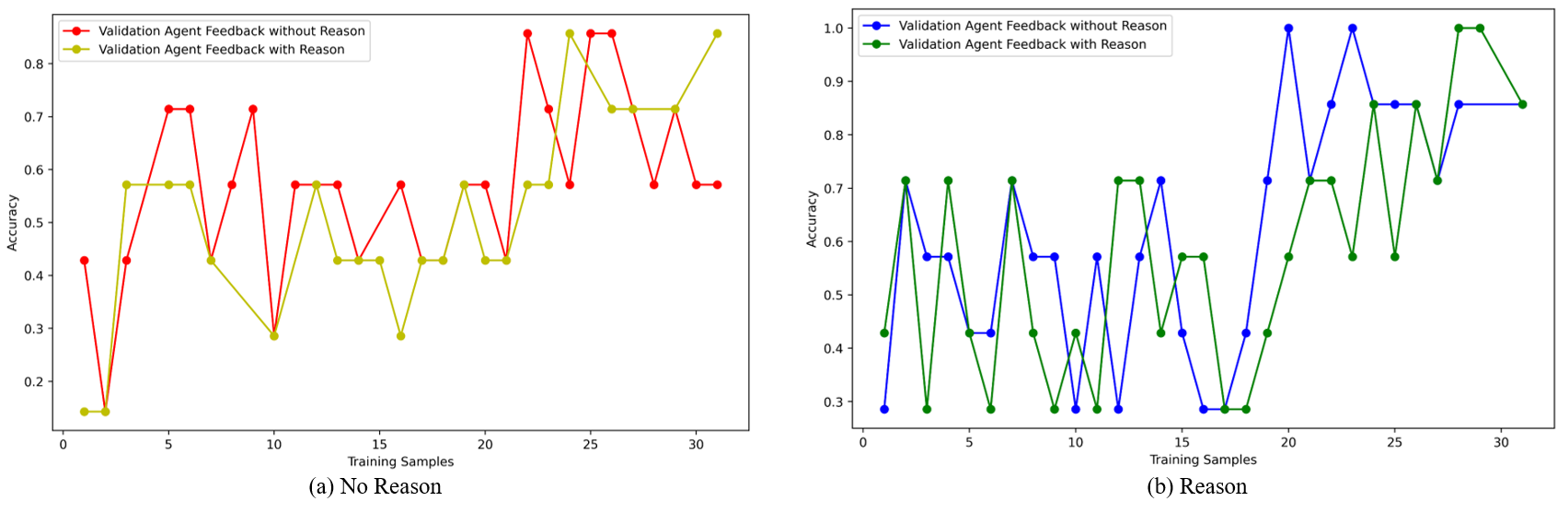}
\caption{{Task2 Experimental Results: Comparative Analysis of Validation Agent Feedback Accuracy Without Reasoning Across Training Samples.}}
\label{header}
\end{figure}

The comparative analysis of these graphs clearly illustrates that models incorporating reasoning in the validation process from experience bases and record libraries perform with enhanced stability and higher accuracy. The performance differential is particularly evident in figure~\ref{header} (b), where the blue line consistently outperforms the red line in Figure~\ref{header} (a), showcasing the benefits of reasoning in achieving more reliable and effective model outputs. This enhanced performance can be attributed to the additional contextual and logical frameworks that reasoning provides, allowing the model to navigate complex data scenarios more effectively.

In the strategy recommendation for task3, the strategy development agent effectively supports strategy development because the input information includes the header events of the event tree. However, due to the excessive length of the text inputs, the GPT generation is unstable, sometimes stopping midway. Therefore, for the evaluation of this task, we focus only on the first step where an accuracy of 100 \% is achieved to assess the model's performance.

As shown in Figure~\ref{strategy} (a), the accuracy values fluctuate significantly throughout the training process for both conditions. Notably, the first instance where the accuracy reaches 1.0 is marked by the yellow series, representing feedback with a reason. This peak occurs at the 11th training sample, highlighting a perfect accuracy achievement under this condition. This critical point indicates a potential advantage in providing reasons in validation agent feedback for achieving optimal accuracy early in the training sequence. The figure~\ref{strategy} (b) depicts that both conditions exhibit considerable fluctuations throughout the training samples, with the accuracy scores oscillating frequently above and below the 0.8 mark. The accuracy reaches 1.0 multiple times in both conditions, but notably, the green line (feedback with reason) first reaches full accuracy in the first training sample. 

\begin{figure}[H]
\centering
\includegraphics[width=0.9\textwidth]{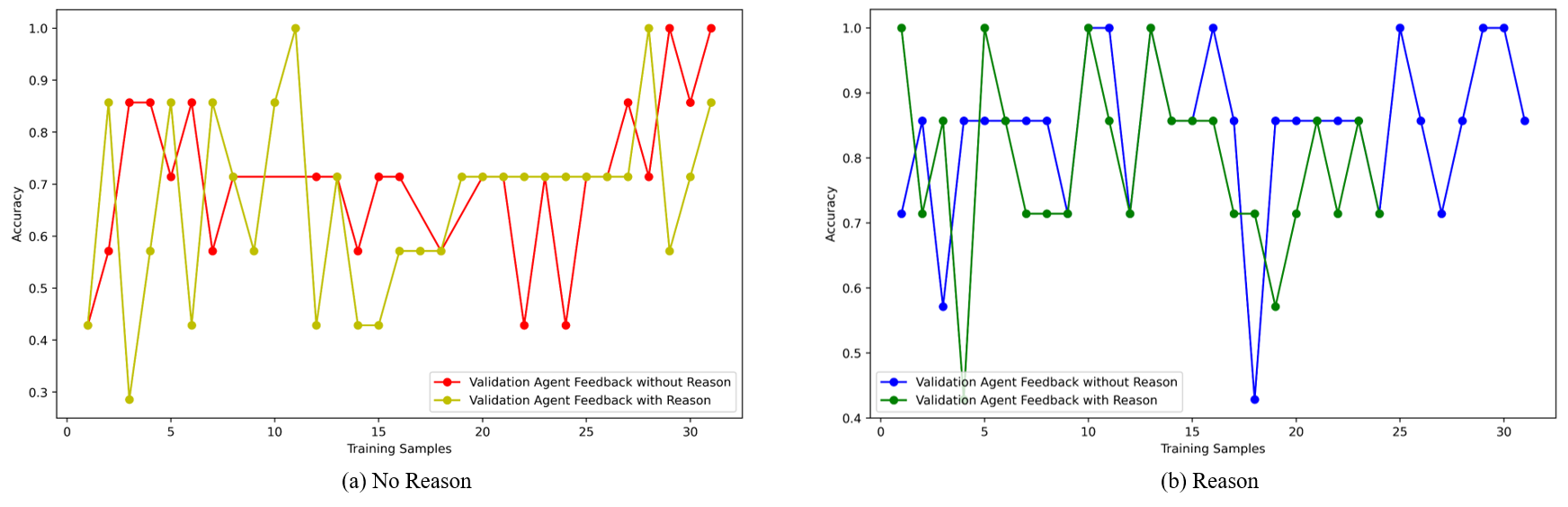}
\caption{{Task 3 Experimental Results: Comparative Analysis of Validation Agent Feedback Accuracy Without Reasoning Across Training Samples.}}
\label{strategy}
\end{figure}

The visualization in figure~\ref{strategy} (b) demonstrates a more stable and consistent high accuracy for feedback with a reason compared to figure~\ref{strategy} (a), suggesting that the inclusion of reasons in the feedback of training data can stabilize and enhance model performance both during the infer stage. In summary, these visualizations effectively argue that providing reasons in validation feedback both in training and validated data is beneficial for immediate infer stage performance. This insight supports the hypothesis that detailed, reason-based feedback fosters a more robust learning environment, potentially leading to better generalization and application in varied scenarios.

\subsection{Case Study}

To better demonstrate the effectiveness of our EvoTaskTree, this section presents a case study. Figure\ref{case} represents a case study describing an initial event involving a main steam line break combined with a steam generator tube rupture accident. Our overall approach employs a decision support method that integrates event tree prior knowledge.

\begin{figure}[H]
\centering
\includegraphics[width=0.9\textwidth]{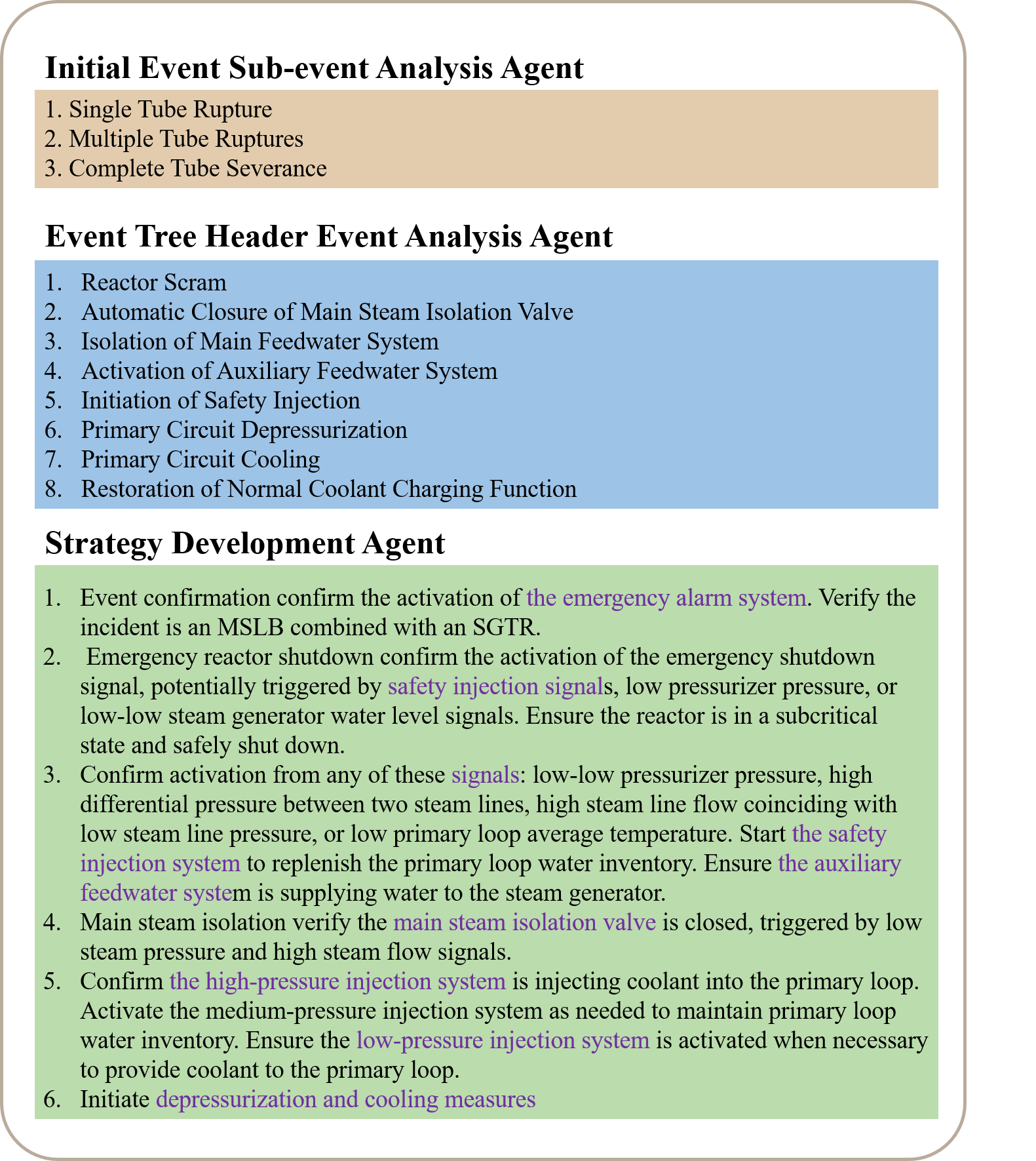}
\caption{{A Case Study for Steam Generator Tube Rupture in EvoTaskTree.}}
\label{case}
\end{figure}

First, the initiating event is fed into task1, represented in the yellow section, which is handled by our initial event subevent analysis agent. This produces a subevent list for the initial event. Next, the output from task1 is used as input for task2, the event tree header event analysis agent. Inputs include the initiating event, description of the initiating event, description of the event, and subevent, resulting in the event tree header event outputs. Finally, these results are input into task3, the strategy development agent, with inputs comprising the initiating event, description of initiating event, description of the event, subevent, event tree header event, and event progression and system response. The output provides the final strategy support.

\section{Discussion and Analysis}

In this section, we compare and discuss the performance of GPT-3.5 and GPT-4o on three tasks. We designed and compared four approaches: Vanilla, CoT, NPPprompt, and our EvoTaskTree. The Vanilla approach involves no prompts, with only a simple task description. The CoT approach builds on Vanilla by adding model-guiding content. The NPPprompt approach incorporates our designed agent prompt information, as detailed in Section \ref{prompt}. The EvoTaskTree approach is based on the framework outlined in Section \ref{pipe}. The evaluation criteria and reference content are consistent with those in Section \ref{inference}. 

Furthermore, for the subsequent tests, we adopted the best strategy derived from Section \ref{Experimental Results}. Specifically, for task1 and task2, we use prompts with reasons but without reasons during the generation process. For task3, we use prompts and generation processes that both include reasons.

As shown in Table ~\ref{task1-compare}, in CoT, due to the limited data with only three items in the test set, the results exhibit considerable variability. For GPT-3.5 no reason, the model generated empty content with too many irrelevant tasks in the results. In contrast, GPT-3.5 reason produced entirely irrelevant content. For GPT-4o no reason, two out of the three tasks were answered correctly, though one task still contained irrelevant content. Conversely, GPT-4o reason performed worse than GPT-4o no reason, with results similar to GPT-3.5 reason, where the content was overly scattered. Additionally, in both task1-prompt and task1-vanilla tests, the primary cause of errors was the overly scattered content that deviated from the correct answers. 

\begin{table}[h]
\caption{Main results of task1 on the test dataset using GPT-3.5 and GPT-4o as the backbones}\label{task1-compare}
\begin{tabular*}{\textwidth}{@{\extracolsep\fill}lcccc}
\toprule%
& \multicolumn{2}{@{}c@{}}{GPT-3.5 } & \multicolumn{2}{@{}c@{}}{GPT-4o } \\\cmidrule{2-3}\cmidrule{4-5}%
Project & No reason & Reason & No reason & Reason\\
\midrule
Vanilla & 33.3333 \% & 66.6667 \% & 33.3333 \% & 0  \\
 CoT & 0 & 0 & 66.6667 \% &  0        \\
   NPPprompt &  33.3333 \% &  0  & 33.3333 \% & 33.3333 \%  \\
EvoTaskTree & 33.3333 \% & 66.6667 \% & 100 \% & 66.6667 \%       \\
\botrule
\end{tabular*}
\footnotetext{Note: Our EvoTaskTree consistently achieved the highest accuracy rates.}

\end{table}

As shown in \ref{task2-compare}, in task2 header event analysis, according to Section \ref{Experience Accumulation}, this is the most challenging task. It requires not only accuracy in content but also correctness in order. This necessitates the model to fully understand the entire progression of the incident. Compared to task1 and task3, the results for task2 are generally not as good.

\begin{table}[h]
\caption{Main results of task2 on the test dataset using GPT-3.5 and GPT-4o as the backbones}\label{task2-compare}
\begin{tabular*}{\textwidth}{@{\extracolsep\fill}lcccc}
\toprule%
& \multicolumn{2}{@{}c@{}}{GPT-3.5 } & \multicolumn{2}{@{}c@{}}{GPT-4o } \\\cmidrule{2-3}\cmidrule{4-5}%
Project & No reason & Reason & No reason & Reason\\
\midrule
  Vanilla & 0 &  0  & 0 & 0  \\
   CoT & 0 & 0 & 0 &  0        \\
   NPPprompt &  0 &  0  & 0 & 42.85714 \%  \\
EvoTaskTree & 28.5714 \% & 42.8571 \% & 71.4286 \% & 85.7143 \%        \\
\botrule
\end{tabular*}
\footnotetext{Note: Our EvoTaskTree consistently achieved the highest accuracy rates.}
\end{table}

For the CoT strategy, GPT-3.5 no reason completely fails to perform the task and does not understand it at all. GPT-3.5 reason fails to recognize that the event tree headings form a sequence and can only analyze a single incident, often generating empty content, but overall, the quality of the generated content is higher than that of GPT-3.5 no reason. GPT-4o no reason misunderstands the task, treating it as an initial event analysis task, but it can generally complete the initial event analysis, although sometimes with incomplete content. GPT-4o reason performs similarly to GPT-4o no reason. Under the prompt strategy, GPT-3.5 no reason understands the task correctly but generates some irrelevant content, such as main steam pipe rupture events and steam generator heat transfer tube rupture events, deviating from the correct answers. The number of generated event headings far exceeds the actual number, sometimes reaching 24 (the correct number of event headings is 9). GPT-3.5 reason performs better than GPT-3.5 no reason, producing more concise and accurate content. GPT-4o no reason generates several event headings closer to the actual number and the content is closer to the real event headings, eliminating a lot of irrelevant descriptions, but it still does not match the correct answers. GPT-4 reason performs better than GPT-4o no reason, even producing three correct cases. For the vanilla strategy, GPT-3.5 no reason does not understand the task and generates no output. GPT-3.5 reason also does not understand the task, with results similar to GPT-3.5 no reason. GPT-4o no reason performs similarly to CoT's GPT-4o no reason. GPT-4o reason performs similarly to GPT-4o no reason but generates more comprehensive content. Our EvoTaskTree consistently achieved the highest accuracy rates. Under both GPT-4-o reason and no reason conditions, it maintained an accuracy rate of 85.7143 \%. Remarkably, under the GPT-4-o reason condition, its accuracy rate is double that of the most effective NPPprompt.

As shown in table ~\ref{task3-compare}, in the task3 testing, the results for the vanilla are shown in Table~\ref{task3-compare}. Firstly, for GPT-3.5 no reason, the output either consists of only the first step of emergency shutdown or is too generic with insufficient steps. In comparison, the GPT-3.5 reason output is not only incomplete but also performs worse than the GPT-3.5 no reason. Additionally, the GPT-4o no reason is highly unstable, with no complete steps. However, it performs better than GPT-3.5 reason overall. On the other hand, the GPT-4o reason output has one correct and complete step, while the rest of the data is similar to the GPT-4o no reason. Ultimately, their respective accuracies are 0\%, 0\%, 0\%, and 14.2857\%. 

\begin{table}[h]
\caption{Main results of task3 on the test dataset using GPT-3.5 and GPT-4 as the backbones}\label{task3-compare}
\begin{tabular*}{\textwidth}{@{\extracolsep\fill}lcccc}
\toprule%
& \multicolumn{2}{@{}c@{}}{GPT-3.5 } & \multicolumn{2}{@{}c@{}}{GPT-4o } \\\cmidrule{2-3}\cmidrule{4-5}%
Project & No reason & Reason & No reason & Reason\\
\midrule
  Vanilla & 0 & 0 & 0 & 14.2857 \%        \\
   CoT & 0 & 28.5714 \% & 0 &  0        \\
   NPPprompt & 57.1429 \% &  71.4286\% & 57.1429 \% & 71.4286 \%        \\
EvoTaskTree & 57.1429 \% & 71.4286 \% & 85.7143\% &  85.7143 \%        \\
\botrule
\end{tabular*}
\footnotetext{Note: Our EvoTaskTree consistently achieved the highest accuracy rates.}
\end{table}

In the context of CoT, the strategy support varies across different models. Under GPT-3.5 no reason, the strategy support task can only provide the first step, which is emergency shutdown. In contrast, GPT-3.5 reason offers relatively more comprehensive support, providing the first and second steps; however, its performance is unstable, sometimes directly outputting according to the initial event description. Specifically, the accuracy rates for GPT-3.5 no reason and GPT-3.5 reason are 0\% and 28.57\%, respectively. For GPT-4o no reason, the outputs are often irrelevant, with an accuracy rate of 0\%, indicating that relying solely on the CoT strategy is insufficient. Nevertheless, GPT-4o demonstrates stronger comprehension abilities compared to GPT-3.5. However, without prompt content, GPT-4o is prone to misunderstandings.

In NPPprompt, GPT-3.5 no reason shows a significant improvement compared to the CoT strategy. Aside from a few instances where the generated steps are insufficient, the responses are generally very accurate and closely match the real results, such as the reference procedure content. In contrast, the error rate for GPT-3.5 reason is higher than that of GPT-3.5 no reason. For GPT-4 no reason, except for the occasional issue of generating only a single sentence, the accuracy of normally generated content can reach 100\%. Meanwhile, GPT-4 reason demonstrates more stable performance compared to GPT-4 no reason. The accuracy rates for these tasks are 57.1429\%, 71.4286\%, 57.1429\%, and 71.4286\%, respectively.

As shown in Figure~\ref{compare}, comparing the three sets of experiments, EvoTaskTree outperformed the other two configurations in task1, task2, and task3. This indicates that integrating data from the record library and the experience base can significantly enhance the algorithm's performance, even elevating the results of the most complex task2 to be comparable to those of the simplest task3.  Furthermore, the configuration using only the record library also showed better performance across most tasks than the one using only the experience base, with the difference being particularly pronounced in task1.

\begin{figure}[H]
\centering
\includegraphics[width=0.9\textwidth]{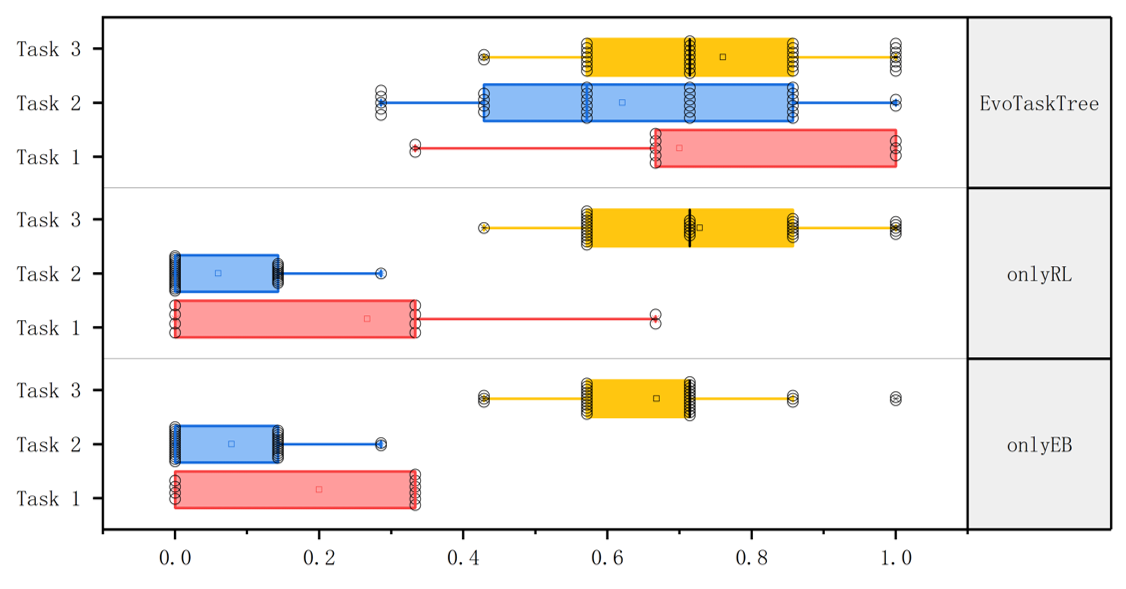}
\caption{{Ablation Study on Multi-Task Learning Performance: Comparing EvoTaskTree with Exclusive Use of Record Library and Experience Base Data}}
\label{compare}
\end{figure}

\section{Conclusion}
In this paper, we integrated the prior knowledge from the event tree analysis in QRA with emergency decision support, proposing a novel approach named EvoTaskTree. It is a LLM-driven simulacrum includes two types of roles (task executors and task validators), and the performance of these agents is evaluated and discussed across three related tasks: initial event subevent analysis, head event analysis, and decision support. This approach encompasses two types of knowledge: one is knowledge of successes, which is stored in the record library, and the other is knowledge of failures, which is preserved in the experience base. The primary conclusions are as follows. 

\begin{itemize}
\item The integrated analysis of event tree analysis and emergency decision support can effectively accomplish emergency decision support. Despite the inherent inaccuracies of LLMs, the accuracy of these tasks can even reach state-of-the-art performance at 100\%. Specifically, if the initial subevent analysis and head event analysis tasks are completed correctly, the strategy support test set can quickly converge to 100\% accuracy.
\item Among the proposed three tasks, the event tree header event analysis is the most challenging. During the training process of this task, EvoTaskTree requires multiple iterations.
\item Compared to other strategies such as Vanilla, CoT, NPPprompt, onlyRL, and onlyEB, EvoTaskTree demonstrates superior performance and significantly outperforms these strategies.
\item Furthermore, we find that the model performs well when the reference experience base and record library include reasons, and the feedback during testing does not include reasons for the initial subevent analysis and head event analysis. Conversely, as for decision recommendations, the model performs better when the reference experience base and record library include reasons, and the feedback during testing also includes reasons.
\end{itemize} 

However, this paper also has some limitations. Future research can consider integrating the fault diagnosis \cite{qi2023, xiao2023, qi2022} process into the entire workflow, thereby making the overall approach more comprehensive.

\section*{Declarations}

\subsection*{Funding}
The research was supported by the Innovation Funds of CNNC–Tsinghua Joint Center for Nuclear Energy R\&D (Project No. 20202009032) and a grant from the National Natural Science Foundation of China (Grant No. T2192933).

\subsection*{Conflict of interest/Competing interests}
The authors declare that they have no conflict of interest.
\subsection*{Ethics approval and consent to participate}
This research did not involve any studies with human participants or animals performed by any of the authors.
\subsection*{Consent for publication}
All authors have reviewed and approved the manuscript for publication. Furthermore, since this study did not involve human participants or animals, specific consent for publication is not applicable.
\subsection*{Data availability}
The datasets generated and/or analyzed during the current study are available from the corresponding author upon reasonable request.
\subsection*{Materials availability}
Not applicable.
\subsection*{Code availability}
The complete code is available on the GitHub website (https://github.com/Crystalxy123/EvoTaskTree/tree/master)

\subsection*{Author contribution}
Xiao Xingyu: Methodology, Software, Formal analysis, Data Curation, Visualization, Validation, Writing- Original draft preparation. Chen Peng: Software, Methodology, Investigation. Qi Ben: Investigation, Writing - Review and Editing. Liang Jingang: Conceptualization, Resources, Supervision, Writing - Review and Editing, Project administration, Funding acquisition. Tong Jiejuan: Investigation, Supervision, Writing - Review and Editing. Wang Haitao: Supervision, Writing- Reviewing and Editing.

\bibliography{sn-bibliography}

\end{document}